\documentclass{article}

    \usepackage[final, nonatbib]{neurips_2020}

\usepackage[utf8]{inputenc} % allow utf-8 input
\usepackage[T1]{fontenc}    % use 8-bit T1 fonts
\usepackage{url}            % simple URL typesetting
\usepackage{booktabs}       % professional-quality tables
\usepackage{amsfonts}       % blackboard math symbols
\usepackage{nicefrac}       % compact symbols for 1/2, etc.
\usepackage{microtype}      % microtypography

\usepackage{subfigure}
\usepackage{cite}

\usepackage{comment}
\usepackage{amsmath,amssymb} % define this before the line numbering.
\usepackage{color}
\usepackage{graphicx}
\usepackage[colorlinks=true]{hyperref}
\usepackage{algpseudocode}
\usepackage{booktabs}
\usepackage{multirow}
\usepackage{pifont}
\newcommand{\cmark}{\ding{51}}%
\newcommand{\xmark}{\ding{55}}%

\usepackage{multirow}

\usepackage[ruled,vlined]{algorithm2e}

\usepackage{amsthm}
\newtheorem{lemma}{Lemma}
\usepackage{float}

\title{Rethinking Learnable Tree Filter for \\ Generic Feature Transform}

\author{Lin Song$^{1}$ 
\quad Yanwei Li$^2$ \quad Zhengkai Jiang$^3$ \quad Zeming Li$^4$ \quad Xiangyu Zhang$^4$ \\ {\bf \quad Hongbin Sun$^1$\thanks{Corresponding author.} \quad Jian Sun$^4$ \quad Nanning Zheng$^1$} \\
$^1$ College of Artificial Intelligence, Xi'an Jiaotong University \\
$^2$ The Chinese University of Hong Kong \\
$^3$ Institute of Automation, Chinese Academy of Sciences \\
$^4$ Megvii Inc. (Face++)\\
stevengrove@stu.xjtu.edu.cn, ywli@cse.cuhk.edu.hk, jiangzhengkai2017@ia.ac.cn, \\ \{hsun, nnzheng\}@mail.xjtu.edu.cn, \{lizeming, zhangxiangyu, sunjian\}@megvii.com}

\begin{document}

\maketitle

\begin{abstract}
% The Learnable Tree Filter, which aggregates context on a minimum spanning tree, presents a remarkable approach to model structural preserving dependencies for semantic segmentation. 
The Learnable Tree Filter presents a remarkable approach to model structure-preserving relations for semantic segmentation. Nevertheless, the intrinsic geometric constraint forces it to focus on the regions with close spatial distance, hindering the effective long-range interactions. To relax the geometric constraint, we give the analysis by reformulating it as a Markov Random Field and introduce a learnable unary term. Besides, we propose a learnable spanning tree algorithm to replace the original non-differentiable one, which further improves the flexibility and robustness. With the above improvements, our method can better capture long-range dependencies and preserve structural details with linear complexity, which is extended to several vision tasks for more generic feature transform. Extensive experiments on object detection/instance segmentation demonstrate the consistent improvements over the original version. For semantic segmentation, we achieve leading performance (82.1\% mIoU) on the Cityscapes benchmark without bells-and-whistles. Code is available at \href{https://github.com/StevenGrove/LearnableTreeFilterV2}{https://github.com/StevenGrove/LearnableTreeFilterV2}.
\end{abstract}

%%%%%%%%% BODY TEXT
\section{Introduction}

In the last decade, the vision community has witnessed the extraordinary success of deep convolutional networks in various vision tasks~\cite{chen2014semantic, landrieu2018large, simonyan2014very, szegedy2015going,song2019tacnet,li2020learning,zhang2019glnet,song2020fine}.
However, in deep convolutional layers, the distribution of impact within an effective receptive field is found to be limited to a local region and converged to the gaussian~\cite{luo2016understanding}, which brings difficulties to the long-range dependencies modeling.
To address this problem, numerous local-based approaches~\cite{chen2017rethinking, chen2018encoder, dai2017deformable} have been proposed to increase the receptive region of convolutional kernels by using pooling~\cite{zhao2017pyramid} or dilated operations~\cite{chen2017deeplab}. Meanwhile, various global-based approaches~\cite{wang2018non, huang2019ccnet, hu2018relation, zhang2019latentgnn, cao2019gcnet} have been explored to aggregate features by modeling the pairwise relations based on the visual attention mechanism. However, there is still a conflict between long-range dependencies modeling and object details preserving.

\begin{figure}[ht]
\includegraphics[width=0.9\textwidth]{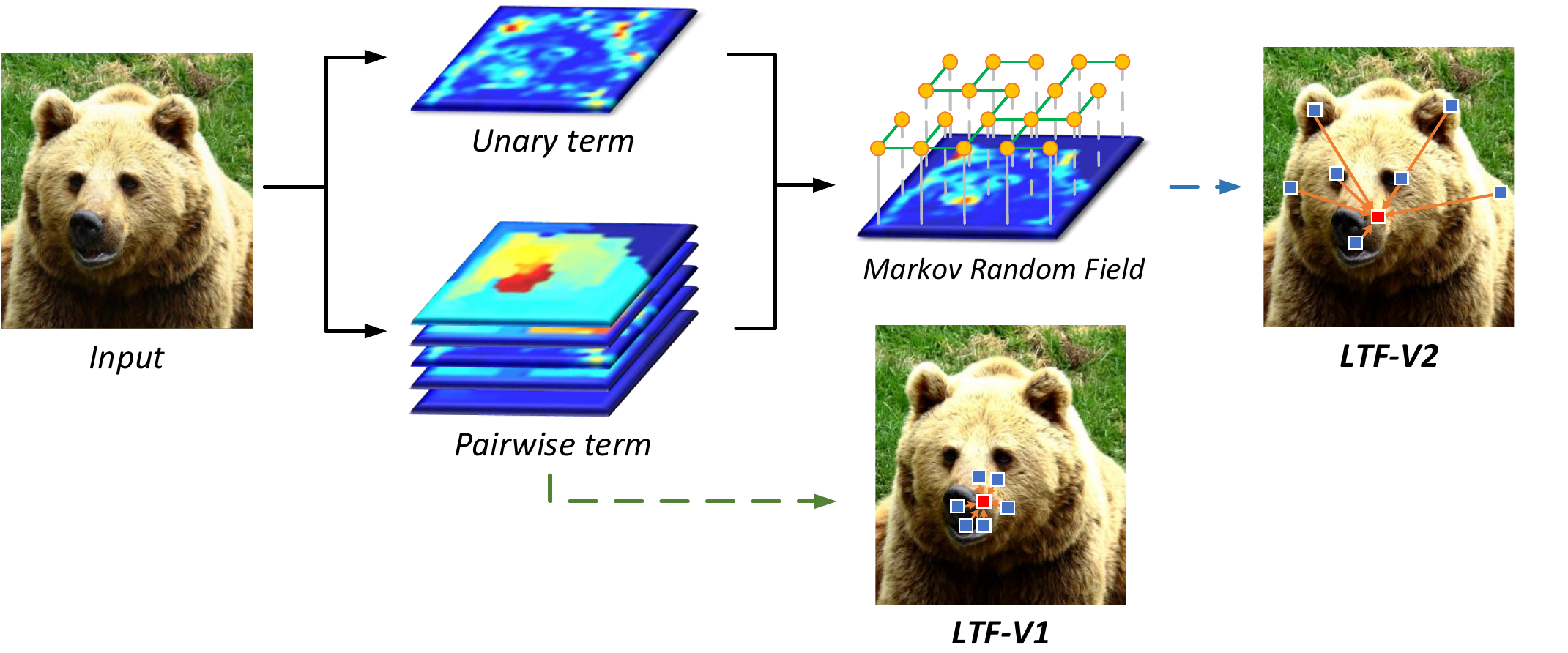}
\centering
\caption{The comparison of the aggregation processes between the LTF-V1 module~\cite{song2019learnable} and the proposed LTF-V2 module. One of the sink nodes is marked by the \textit{red} square, and the corresponding source nodes with the {\em highest} filtering weights are presented as the \textit{blue} squares.
Limited by the intrinsic geometric constraint, the LTF-V1 can only give priority to the regions with close spatial distance.
Different from it, our proposed LTF-V2 is more generic by adding a learnable unary term, which relaxes the geometric constraint and allows the filter to focus on the distant regions of interest.}
\label{fig:vis}
\end{figure}

Recently, the learnable tree filter module~\cite{song2019learnable} (LTF-V1 module) tries to bridge this gap by performing feature aggregation on a minimum spanning tree.
Since the minimum spanning tree is generated by the low-level guided features, it can retain the structural details with efficiency.
% Nevertheless, the {\em geometric constraint} and the {\em construction process} in LTF-V1 are found to be the Achilles' heel , which impede the usage for more generic feature transform.
% Firstly, due to the nature of LTF-V1, interactions along the whole spanning tree are dominated by the nearby nodes, which brings the intrinsic \textit{geometric constraint} to the tree filter, as presented in Fig.~\ref{fig:vis}. Moreover, the construction process of minimum spanning tree is non-differentiable, which prevents the LTF-V1 module from entirely learnable.
Nevertheless, the {\em geometric constraint} and the {\em construction process} in the LTF-V1 module are found to be its Achilles' heels, which impede the usage for more generic feature transform.
Firstly, the interactions with distant nodes along the spanning tree need to pass through nearby nodes, which brings the intrinsic {\em geometric constraint} to the tree filter. As illustrated in Fig.~\ref{fig:vis}, this property forces the LTF-V1 module to focus on the nearby regions. In addition, the {\em construction process} of the minimum spanning tree is non-differentiable, which is overly sensitive to the quality of guided features and prevents the LTF-V1 module from entirely learnable.

% As presented in the left most of Fig.~\ref{fig:vis}, the LTF-V1 is forced to aggregate the features with close spatial distance preferentially. 

% Through the theoretical analysis, we reformulate the LTF-V1 as an equivalent form of Markov Random Field~\cite{li1994markov} (MRF), whose unary term is constant. Thus, the filtering weight monotonously decreases as the distance increases along a path in the tree.~\footnote{The proof is provided in the supplementary material.\label{proof}}

To remedy the shortages mentioned above, we rethink the learnable tree filter from the perspective of Markov Random Field~\cite{li1994markov} (MRF) and present the {\em Learnable Tree Filter V2 Module} (LTF-V2 module) for more generic feature transform. Specifically, it complements the LTF-V1 module by introducing a {\em learnable unary term} and a {\em learnable spanning tree algorithm}. 
The former one provides a modulation scalar~\cite{poggi2017quantitative, tosi2018beyond, park2015leveraging} for each node, which can relax the geometric constraint and enable effective long-range interactions.
% {\em i.e.}, allowing the distant nodes to have larger filtering weight over the close ones along a path in the tree.
Intuitively, as shown in the Fig.~\ref{fig:vis}, the proposed unary term guides the LTF-V2 to focus on the coarsely distant regions of interest, and the tree-based pairwise term further refines the regions to fit original structures, bringing powerful semantic representations.
Meanwhile, the proposed learnable spanning tree algorithm offers a simple and effective strategy to create a gradient tunnel between the spanning tree process and the feature transform module.
This algorithm enables the LTF-V2 module to be more robust and flexible.
Moreover, the LTF-V2 module maintains linear computational complexity and highly empirical efficiency of the LTF-V1~\footnote{The runtime benchmark is provided in the supplementary material.\label{runtime}}.

Overall, in this paper, we present a more generic method for structure-preserving feature transform with linear complexity.
To demonstrate the effectiveness of the LTF-V2 module, we conduct extensive ablation studies on object detection, instance segmentation, and semantic segmentation.
Both quantitative and qualitative experiments demonstrate the superiority of the LTF-V2 module over previous approaches.
Even compared with other state-of-the-art works, the LTF-V2 module achieves competitive performance with much less resource consumption.
Specifically, when applied to Mask R-CNN~\cite{he2017mask} (with ResNet50-FPN~\cite{he2016deep}), the LTF-V2 module obtains 2.4\% and 1.8\% absolute gains over the baseline on COCO benchmark for $\rm{AP}^{box}$ and $\rm{AP}^{seg}$ respectively, with negligible computational overhead.
Meanwhile, the LTF-V2 module achieves \textbf{82.1}\% mIoU on Cityscapes benchmark without bells-and-whistles, reaching leading performance on semantic segmentation.

\section{Related Work}
Since each neuron has limited effective receptive field, the deep convolutional networks~\cite{chen2014semantic, simonyan2014very, he2016deep} typically fail to capture long-range dependencies.
Recently, to alleviate the problem, a large number of approaches have been proposed to model long-range context in existing deep convolutional models.
These approaches can be generally divided into two categories, namely local-based and global-based.

The local-based approaches aim to aggregate the long-range context by increasing the local receptive region of each neuron, {\em e.g.}, pooling operation and dilated convolution.
He \textit{et al.}~\cite{he2015spatial} and Zhao \textit{et al.}~\cite{zhao2017pyramid} propose a spatial pyramid pooling for object detection and semantic segmentation, respectively.
Yu \textit{et al.}~\cite{yu2015multi} introduces the dilated convolution to enlarge the receptive field of each convolution kernel explicitly.
Dai \textit{et al.}~\cite{dai2017deformable} further modifies the dilated convolution to a more generic form by replacing the grid kernels with the deformable ones.
However, the modeling of pairwise relation by these methods is typically restricted to a local region and relies on the homogeneous prior.

The global-based approaches are mainly based on the attention mechanism, which is firstly applied in machine translation~\cite{vaswani2017attention} as well as physical system modeling~\cite{battaglia2016interaction} and then extended to various vision tasks~\cite{wang2018non}.
SENet~\cite{yu2018bisenet}, PSANet~\cite{zhao2018psanet}, and GENet~\cite{hu2018gather} bring channel-wise relations to the network by performing down-sampling and attention in different channels.
Non-Local~\cite{wang2018non} adopts the self-attention mechanism in the spatial domain to aggregate related features by generating the affinity matrix between each spatial node. 
It can model non-local relations but suffers from highly computational complexity.
LatentGNN~\cite{zhang2019latentgnn} and CCNet~\cite{huang2019ccnet} are proposed to alleviate this problem by projecting features into latent space and stacking two criss-cross blocks, respectively.

Nevertheless, with the expansive receptive field, the detailed structure could not be preserved in the above methods.
Our method bridges this gap by taking advantages of structure-aware spanning trees and global modeling of Markov Random Field.

% \subsection{Markov Random Field}
% Markov Random Field (MRF) or Conditional Random Field (CRF) are elegant graphical models for a set of random variables that satisfy the Markov property.
% At the same time, the MRF/CRF are also powerful tools and have achieved great successes in various vision tasks ~\cite{li1994markov, boykov2001fast, sun2003stereo, felzenszwalb2006efficient, liu2019learning}.
% Recently, many approaches attempt to combine the deep neural network with MRF/CRF.
% DeepLab~\cite{chen2017deeplab} utilizes a fully-connected CRF as a post-processing step to refine the predictions of the network. It can recover the details while maintaining the smoothness of the homogeneous area.
% Schwing \textit{et al.}~\cite{schwing2015fully} and Arnab \textit{et al.}~\cite{arnab2016higher} propose strategies to jointly train convolutional neural networks (CNNs) and MRF in a unified framework.
% CRF-as-RNN~\cite{zheng2015conditional} and Deep Parsing Network (DPN)~\cite{liu2017deep} further adopt the recurrent neural networks (RNNs) and CNNs to replace the mean-field approximate process for CRFs and MRFs, respectively.
% These approaches can be trained in an end-to-end manner and achieve promising performance.

% However, the above methods still have a limitation that they can only be utilized at the end of the network as a refinement for predictions.
% Different from them, our proposed method takes advantage of the MRF and can be used as a plug-and-play module in CNNs.

\section{Method}
In this section, we theoretically analyze the deficiency of the LTF-V1 by reformulating it as a Markov Random Field and give a solution, namely the LTF-V2.
Besides, we propose a learnable spanning tree algorithm to improve robustness and flexibility.
With these improvements, we further present a new learnable tree filter module, called LTF-V2 module.
%To address the problem, 
%In this section, we first give the problem definition and review of the LTF-V1.
%In this section, we first introduce some notations and problem definitions. Then the analysis of the LTF-V1 is presented. It can be reformulated as a specific form of MRF. To solve the limitation of the LTF-V1, a more generic version (LTF-V2) is proposed by modeling of unary term. Next, learnable random spanning tree is introduced to make it fully differentiable. At last, network architecture is presented in detail.

\subsection{Notation and Problem Definition }
% Modeling long-range dependencies in computer vision tasks can be interpreted as learning a feature transform process to obtain a globally structured representation of the feature map~\cite{song2019learnable, wang2018non,hu2018relation,zhang2018context}.
Given an input feature map ${X}=\{x_i\}^N$ with $x_i\in \mathbb{R}^{1\times C}$ and the corresponding guidance ${G}=\{g_i\}^N$ with $g_i\in \mathbb{R}^{1\times C}$, where $N$ and $C$ indicate the number of input pixels and encoding channels, respectively. 
Specifically, the guidance ${G}$ gives the positions where the filter needs to preserve detailed structures.
Following the original configuration~\cite{yang2015stereo,song2019learnable}, we represent the topology of the guidance ${G}$ as a 4-connected planar graph, \textit{i.e.}, $\mathcal{G}=\{\mathcal{V}, \mathcal{E}\}$, where $\mathcal{V}=\{\mathcal{V}_i\}^N$ is the set of nodes and $\mathcal{E}=\{\mathcal{E}_i\}^N$ is the set of edges.
The weight of the edge reflects the feature distance between adjacent nodes.
Our goal is to obtain the refined feature map ${Y}$ by transforming the input feature map ${X}$ with the property of guided graph $\mathcal{G}$, where the dimension of ${Y}$ is same with that of ${X}$.
% In the graph domain, we define the element of input feature map as observable variable and denote a set of latent variables as $\mathbf{H}=\{h_i\}^N$ with $h_i\in\{x_1,x_2,...,x_N\}$, whose element $h_i$ has the same dimension as the corresponding observable variable $x_i$. For quantitative analysis, we consider the establishment of long-range dependencies as a stochastic sampling process from the set of observable variables.
% Therefore, the feature transform on the input feature map is equivalent to calculating the statistical expectation of latent variables.
% It can be further formulated as Eq.~\ref{eq:y_i}, where $y_i$ is the transformed output of the corresponding node and $P_\mathcal{G}(h_i=x_j)$ is the marginal probability.
% \begin{equation}
% \label{eq:y_i}
% y_i=E(h_i)=\sum_{\forall x_j\in {X}}{P_\mathcal{G}(h_i=x_j) x_j}.
% \end{equation}
% The marginal probability of latent variables is controlled by the topological structure and weights of the graph~\cite{scarselli2008graph}.
% To this end, our target is to construct a learnable process to build the graph $\mathcal{G}$ and a likelihood function $P_\mathcal{G}(\cdot)$ to map the latent variables and the graph into a normalized scalar.
\subsection{Revisiting the Learnable Tree Filter V1}
% \begin{figure}[t]
% \centering
% \includegraphics[width=8.0cm]{Fig/tree_filter_v1.pdf}
% \caption{The diagram of the procedure of the LTF-V1 module. First, a minimum spanning tree (\textbf{MST}) $\mathcal{G}_T$ is sampled from the 4-connected grid graph $\mathcal{G}$ according to the edge weights. Second, we treat each node (\textit{e.g.}, the circle in \textit{green}) as the root, and then aggregate the features from other nodes (\textit{e.g.}, the circle in \textit{red}) along the paths in the tree.
% Nevertheless, due to the geometric constraint and the constant modeling of the unary term, the filtering weight will decrease as the distance from the root node increases (shown by the size and saturation of the arrow)}
% \label{fig:tf_v1}
% % \vspace{-5pt}
% \end{figure}
% In many low-level computer vision tasks, the traditional tree filter~\cite{yang2015stereo} provides a remarkable solution for denoising and distance transform~\cite{bao2013tree, tu2016real} with linear computational complexity.
Recently, a learnable tree filter module~\cite{song2019learnable} (LTF-V1 module) has been proposed as a flexible module to embed in deep neural networks.
The procedure of the LTF-V1 module can be divided into two steps (the visualization is provided in the supplementary material).

First, a tree-based sparse graph $\mathcal{G}_T$ is sampled from the input 4-connected graph $\mathcal{G}$ by using the minimum spanning tree algorithm~\cite{kruskal1956shortest}, as shown in Eq.~\ref{eq:mst}.
The pruning of the edges $\mathcal{E}$ gives priority to removing the edge with a large distance so that it can smooth out high-contrast and fine-scale noise while preserving major structures.
\begin{equation}
\label{eq:mst}
\mathcal{G}_T\sim \mathrm{MinimumSpanningTree}(\mathcal{G}).
\end{equation}
Second, we iterate over each node, taking it as the root of the spanning tree $\mathcal{G}_T$ and aggregating the input feature ${X}$ from other nodes.
% in the same form as the problem definition, the output of LTF-V1 is the marginal expectation of the corresponding latent variable whose probability uses the similarities of the paths on the tree. 
For instance, the output of the node $i$ can be formulated as Eq.~\ref{eq:tf}, where $z_i$ is the normalizing coefficient and ${E}_{j,i}$ is the set of edges in the path from node $j$ to node $i$.
Besides, $\omega_{k,m}$ indicates the edge weight between the adjacent nodes $k$ and $m$. $S_{\mathcal{G}_T}(\cdot)$ accumulate the edge weights along a path to obtain the \textit{filtering weight}, \textit{e.g.}, $\frac{1}{z_i}S_{\mathcal{G}_T}(\cdot)$ for the node $i$.
\begin{equation}
\label{eq:tf}
y_i=\frac{1}{z_i} \sum_{\forall j \in \mathcal{V}} {S_{\mathcal{G}_T}\left({E}_{j,i}\right)x_j},
\quad {\rm{where}}\ S_{\mathcal{G}_T}({E}_{j,i})=\exp (-\sum_{\forall (k,m)\in {E}_{j,i}}{\omega_{k,m}}).
% \prod_{\forall (k,m)\in {E}_{i,j}}{\exp (-\omega_{k,m})}.
\end{equation}
In the original design~\cite{song2019learnable}, the edge weight is instantiated as the euclidean distance between features of adjacent nodes, \textit{e.g.}, $\omega_{k,m}=|x_k-x_m|^2$. 
In addition, due to the tree structure is acyclic, we can use dynamic programming algorithms to improve computational efficiency significantly.

% Due to the absence of loops in tree structure, a specific dynamic programming algorithm is designed to reduce the computational complexity to linear.
% Different from several similar approaches~\cite{wang2018non,hu2018relation}, the efficient implementation allows it to be used in large-scale node applications.
% In brief, LTF-V1 proposes a modeling approach for marginal distribution $P_\mathcal{G}(\cdot)$ that utilizes the property of minimum spanning trees to preserve structure details.
\subsection{Learnable Tree Filter V2}
\label{sec:ab}
Extensive experiments~\cite{song2019learnable} demonstrate the effectiveness of LTF-V1 module on semantic segmentation.
% To verify the generalization capability, we conduct several experiments on object detection and instance segmentation with the LTF-V1 module.
However, the performance of the LTF-V1 module in the instance-aware tasks is unsatisfactory, which is inferior to some related works (\textit{e.g.}, LatentGNN~\cite{zhang2019latentgnn} and GCNet~\cite{cao2019gcnet}).
The details are presented in Tab.~\ref{tab:det_stage}.
We speculate that the reason can be attributed to the intrinsic geometric constraint.
In this section, we try to analyze this problem and give the solution.

\textbf{Modeling.}
According to the Eq.~\ref{eq:tf}, we can consider the LTF-V1 as the statistical expectation of the sampling from input features under a specific distribution. The distribution is calculated by using the spanning tree.
% , \textit{e.g.}, the probability for node $i$ is $\frac{1}{z_i}S_{\mathcal{G}_T}(\cdot)$.
To be more generic, we first define a set of random latent variables as $\mathbf{H}=\{h_i\}^N$ with $h_i\in \mathcal{V}$.  
% define a set random observable variables as $\mathbf{O}=\{o_i\}^N$ with $o_i\in \mathbb{R}^{1\times C}$, which associates to the input feature.
% denote the elements of input feature map ${X}$ as a set of observable variables, 
% The marginal probability of latent variables is denoted as $P_{\mathcal{G}_T}$.
And we give a more generic form of feature transform, which is elaborated in Eq.~\ref{eq:y_i}.
For instance, when the probability distribution $P_{\mathcal{G}_T}$ is set to the filtering weight of the LTF-V1, it will be equivalent to the LTF-V1.
% Specifically, it can degrade to the LTF-V1 when $P_{\mathcal{G}_T}(h_i=x_j)=\frac{1}{z_i}S_{\mathcal{G}_T}({E}_{j,i})$.
% The statistical expectation of latent variables, as formulated in Eq.~\ref{eq:y_i}, is formally equivalent to the LTF-V1 process when $P_{\mathcal{G}_T}(h_i=x_j)=\frac{1}{z_i}S_{\mathcal{G}_T}({E}_{j,i})$.
% We find that the feature transformation on the input feature map is formally equivalent to calculating the statistical expectation of latent variables.
% It can be formulated as Eq.~\ref{eq:y_i}, where $P_{\mathcal{G}_T}(h_i=x_j)$ is the marginal probability.
% Therefore, we consider the learnable tree filter as a stochastic sampling process from a set of observable variables.
\begin{equation}
\label{eq:y_i}
y_i=\mathbb{E}_{h_i\sim P_{\mathcal{G}_T}}[x_{h_i}]=\sum_{\forall j\in \mathcal{V}}{P_{\mathcal{G}_T}(h_i=j) x_j}.
\end{equation}
Moreover, we adopt the Markov Random Field (MRF)~\cite{li1994markov} to model the distribution $P_{\mathcal{G}_T}$, which is a powerful tool to model the joint distribution of latent variables by defining both unary and pairwise terms.
The formulation of the MRF is shown in Eq.~\ref{eq:p_h}, where $\mathbf{O}$ is the set of random observable variables associating with the input feature. $\phi_i$ and $\psi_{i,j}$ represent unary and pairwise terms of the MRF, respectively. Besides, ${Z}$ indicates the partition function.
\begin{equation}
\label{eq:p_h}
P_{\mathcal{G}_T}(\mathbf{H}|\mathbf{O}={X})=\frac{1}{{Z}}\prod_{\forall i \in \mathcal{V}}{\phi_i(h_i, x_i)}\prod_{\forall (i, j)\in \mathcal{E}}{\psi_{i,j}(h_i, h_j)}.
\end{equation}
% The marginal probability of latent variables is controlled by the topological structure and weights of the graph~\cite{scarselli2008graph}.To this end, our target is to construct a learnable process to build the graph $\mathcal{G}$ and a likelihood function $P_\mathcal{G}(\cdot)$ to map the latent variables and the graph into a normalized scalar.

\begin{figure}[t]
\centering 
\subfigure[Network Architecture]
{                    
\begin{minipage}[b]{0.5\textwidth}
\centering                                         
\includegraphics[width=\textwidth]{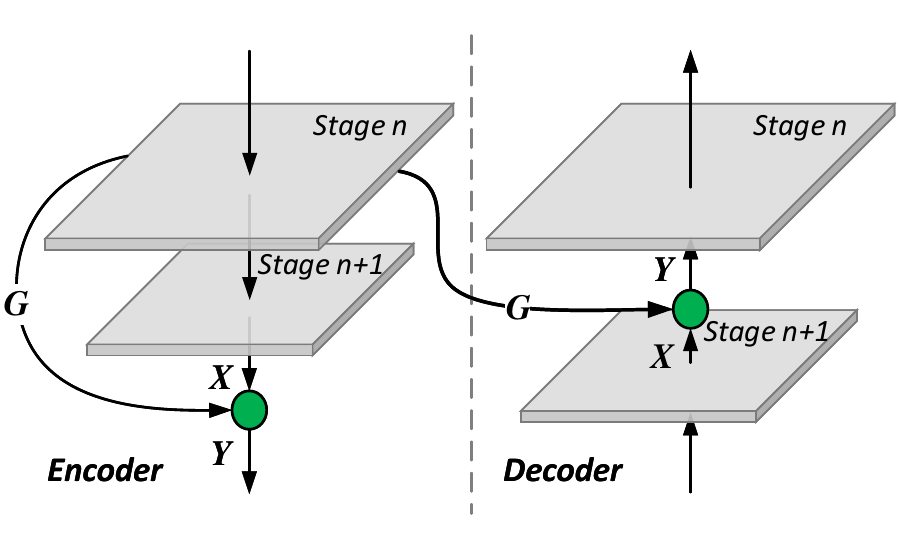}
\end{minipage}
}
\subfigure[Learnable Tree Filter V2 Module]
{                    
\begin{minipage}[b]{0.45\textwidth}
\centering                                         
\includegraphics[width=\textwidth]{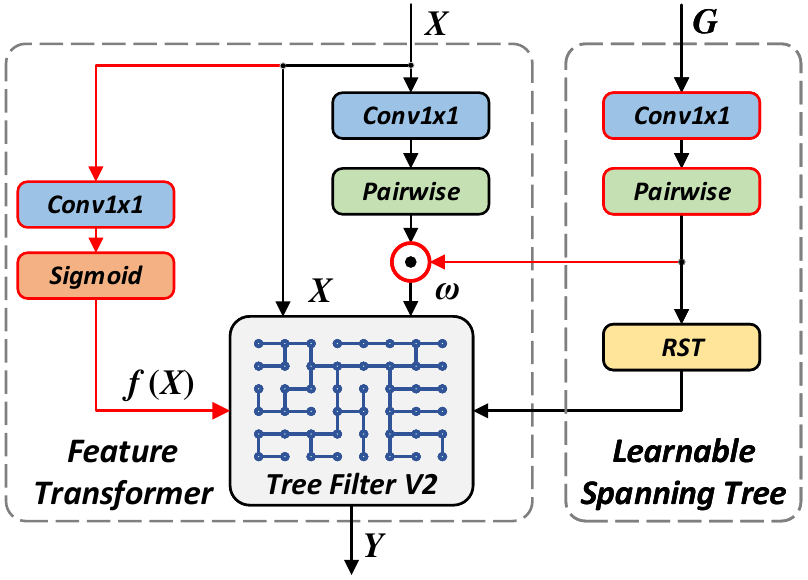}
\end{minipage}
}
\caption{The diagram of the network architecture and the proposed Learnable Tree Filter V2 module. The \textit{green} circles in (a) denote the LTF-V2 modules, which generate the spanning tree with the low-level feature (\textit{i.e.}, “Stage n”) and then transform the high-level feature (\textit{i.e.}, “Stage n+1”). We use \textit{red} in (b) to represent the newly added components compared with the LTF-V1 module. 
% It is composed of a feature transformer and a learnable spanning tree process, which are represented by the left and right dotted rectangles, respectively. 
$f({X})$ and $\mathbf{\omega}$ are the elements in unary and pairwise terms, respectively. ``RST'' denotes the proposed close random spanning tree algorithm. And ``Pairwise'' indicates calculating the distance for each edge.}
\label{fig:tf_v2}
\end{figure}

\textbf{Analysis.} 
First of all, we analyze why the LTF-V1 module works well on semantic segmentation.
By comparing the formulas, we give a specific MRF on tree $\mathcal{G}_T$, which is formally equivalent to the LTF-V1 (the proof is provided in the supplementary material).
% ~\footnote{Detailed proofs are provided in the supplementary material.\label{proof}}.
Specifically, the formulation is shown as Eq.~\ref{eq:phi} and Eq.~\ref{eq:psi}, where $\delta(\cdot)$ denotes the unit impulse function and ${\rm Desc}_{\mathcal{G}_T}(i,j)$ is the descendants of node $j$ when node $i$ is the root of tree $\mathcal{G}_T$.

\begin{equation}
\label{eq:phi}
\phi_i(h_i, x_i)\equiv 1,
\end{equation}
\begin{equation}
\label{eq:psi}
\psi_{i,j}(h_i, h_j):=\left\{\begin{matrix}
\delta(h_i - h_j) & h_i\notin 
{\rm Desc}_{\mathcal{G}_T}(i, j) \\ 
{\rm exp}(-\omega_{i,j}) \delta(h_i - h_j) & h_i\in {\rm Desc}_{\mathcal{G}_T}(i, j)
\end{matrix}\right.
\end{equation}
The MRF is demonstrated to preserve the local structure by modeling data-dependent pairwise terms between adjacent nodes in previous works~\cite{li1994markov, shi2000normalized, boykov2001fast, sun2003stereo, felzenszwalb2006efficient}.
This property of MRF is extremely desired by semantic segmentation~\cite{chen2018deeplab, liu2017deep, zheng2015conditional, yu2018learning}.
Furthermore, combined with the visualizations with rich details~\cite{song2019learnable}, we consider the reason for the effectiveness of LTF-V1 on semantic segmentation lies in modeling learnable pairwise terms.

% \begin{figure}
% \centering
% \includegraphics[width=8.5cm]{Fig/Introduction.pdf}
% \caption{Diagram of the tree filter V2. Compared with the original algorithm, the tree filter V2 models more generalized Markov Random Field (MRF) with trainable unary terms and adopts a learnable spanning tree process (LST).}
% \label{fig:method}
% \end{figure}

We further analyze why the LTF-V1 module performs unsatisfactorily in the instance-aware tasks.
The effective long-range dependencies modeling, which provides powerful semantic representations, is proved to be crucial for instance-aware tasks~\cite{hu2018relation, wang2018non, zhang2019latentgnn, huang2019ccnet}.
However, due to the intrinsic geometric constraint of the LTF-V1, as illustrated in Fig.~\ref{fig:vis}, interactions with distant nodes need to pass through nearby nodes.
Meanwhile, as shown in Eq.~\ref{eq:phi}, the LTF-V1 models the unary term as a constant scalar, resulting in the filtering weight monotonously decreases as the distance increase along a path (the proof is provided in the supplementary material).
% (the proof is provided in the supplementary material).
% all nodes of the input feature being considered as spatially equivalent~\cite{gross2005graph}.
% The adaptive modeling of unary term can relax such intrinsic geometric constrain by direct modulating the interaction.
Therefore, these properties force the LTF-V1 to focus on the \textbf{nearby region}, leading to the difficulty of long-range interactions and unsatisfactory performance in instance-aware tasks.
% Therefore, we speculate that the \textbf{constant modeling of the unary term makes the LTF-V1 hard to handle long-range interactions} (proof is provided in the supplementary material), which results in unsatisfactory performance in the instance-aware tasks.

% The instance-aware tasks require features to be spatially hierarchical, which encourages the algorithm to output the optimal predictions by suppressing the suboptimal predictions.
% Therefore, many works use rescoring~\cite{jiang2018acquisition}, attention mechanism~\cite{hu2018relation} and adjusting the training bias~\cite{cai2018cascade} to enhance the spatial hierarchy of features.
% The experimental results show that it is very important to the performance in instance-aware tasks, especially for the scenes requiring high localization accuracy.
% However, as illustrated on Eq.~\ref{eq:phi}, LTF-V1 models the unary term as constant one, resulting in all nodes of the input feature being considered as spatially equivalent.
% In other words, 
% In other words, LTF-V1 adopts the prior knowledge (\textit{i.e.}, modeling pairwise term) to smooth the input features, without considering the posterior confidence~\cite{poggi2017quantitative} (\textit{i.e.}, reliability of the feature) of each node.

\textbf{Solution.} 
Eventually, we try to give a solution to address the problem.
To focus on the distant region, we need to relax the geometric constraint and allow the filtering weight of distant nodes to be larger than that of close ones.
% In theory, there exists a form of unary term such that it allows the filtering weight of the distant node is larger than the filtering weight of the close node along a path in the spanning tree.
In this paper, motivated by confidence estimation methods~\cite{poggi2017quantitative, tosi2018beyond, park2015leveraging}, we present a learnable unary term which can meet the requirements (the proof is provided in the supplementary material).
% ~\textsuperscript{\ref{proof}}.
%
% Since the adaptive modeling of unary term can relax such intrinsic geometric constraint by direct modulating the interaction~\cite{gross2005graph, zhu2018deformable}, we propose a learnable form of unary term.
The formulation is shown on Eq.~\ref{eq:f_x}, where $f(\cdot)$ denotes a unary function embedded in the deep neural network and $\beta$ is a learnable parameter.
Intuitively, for a node, $f(\cdot)$ reflects the confidence of its input feature, while $\beta$ is the potential for choosing other features.
\begin{equation}
\label{eq:f_x}
\phi_i(h_i, x_i) := \left\{\begin{matrix}
f(x_i) & h_i=i\\ 
{\rm exp}(-\beta) & h_i\neq i
\end{matrix}\right.
\end{equation}
Let Eq.~\ref{eq:f_x} and Eq.~\ref{eq:psi} substitute into Eq.~\ref{eq:p_h}.
Since the tree $\mathcal{G}_T$ is a acyclic graph, the closed-form solution of Eq.~\ref{eq:y_i} can be derived as Eq.~\ref{eq:tf_v2} by using the belief propagation algorithm~\cite{yedidia2001generalized}. $z_i$ is the normalizing coefficient and $|{E}_{j,i}|$ denotes the number of edges in the path.
\begin{equation}
\label{eq:tf_v2}
\begin{aligned}
&y_i=\frac{1}{z_i} \sum_{\forall x_j \in {X}} {S_{\mathcal{G}_T}({E}_{j,i}){{\rm exp}(-\beta)}^{|{E}_{j,i}|}f(x_j)x_j}.\\
% &\mathrm{where}\ z_i= \sum_{\forall x_j \in {X}} {S_{\mathcal{G}_T}\left({E}_{j,i}\right){\rm{exp}(-\beta)}^{|{E}_{j,i}|}f(x_j)}.
\end{aligned}
\end{equation}
Accordingly, we define the Eq.~\ref{eq:tf_v2} as a more generic form of the learnable tree filter, namely the \textbf{Learnable Tree Filter V2 (LTF-V2)}.
The unary and pairwise terms in the LTF-V2 are complementary and promoting to each other.
The unary term on Eq.~\ref{eq:f_x} allows the LTF-V2 to focus on the \textbf{distant region}, bringing more effective long-range interactions. Meanwhile, the data-dependent pairwise term can further refine the distant region to fit detailed structures.
% and significantly improves the model capability and generalization ability. 
% Compared with the LTF-V1, the new unary term brings in modulation scalar $f(x_j)$ and geometric distance ${\rm{exp}(-\beta)}^{|{E}_{j,i}|}$, which significantly improves the model capability and generalization ability.

\subsection{Learnable Random Spanning Tree}
\label{sec:lrst}
\begin{center}
% \begin{minipage}{.8\linewidth}[t]
\begin{algorithm*}[t]
\label{alg:rst}
\caption{Close random spanning tree}
\LinesNumbered
\KwIn{A 4-connected graph $\mathcal{G}$.}
\KwOut{Random spanning tree $\mathcal{G}_T$.}
$\mathcal{G}_T\leftarrow \emptyset$.\\
\For{$e\in E(\mathcal{G})$}
{
    $l(e)\leftarrow e.$ \Comment{Initialize a label for each edge.}
}
\While{$|V(\mathcal{G})| > 1$}
{
    \For{$v_i \in V(\mathcal{G})$}
    {
        \textcolor{blue}{$e_i\sim E_{\mathcal{G}}(v_i)$.\Comment{Sample an incident edge.}}\\
        % \textcolor{blue}{$p_i\leftarrow e_i/\sum{E_{\mathcal{G}}(v_i)}$.\Comment{Sampling probability.}}\\
        $\mathcal{G}_T\leftarrow \mathcal{G}_T\cup \{l(e_i)\}$.
    }
    $\mathrm{Contract}(E(\mathcal{G}))$.\Comment{Contraction algorithm.}\\
    $\mathrm{Flatten}(\mathcal{G})$.\Comment{Remove loops and parallel edges.}
}
\Return $\mathcal{G}_T$.
\end{algorithm*}
% \end{minipage}
\end{center}
Although the LTF-V1 module takes the first step to make the tree filtering process~\cite{yang2015stereo} trainable, the inside minimum spanning tree algorithm is still not differentiable.
This problem prevents it from being entirely learnable. 
Moreover, the topology of the spanning tree is determined by guided features only.
% Therefore, the response of the LTF-V1 module is sensitive to the quality of the guided feature.
Therefore, the initialization and source selection of guided features could significantly impact the performance in the original design.
In this paper, we bridge this gap by proposing a simple architecture and a {\em close random-spanning-tree} algorithm, which is briefly introduced in Fig.~\ref{fig:tf_v2}.

% Reinforcement learning~\cite{mnih2015human} is a widely used way to make the discrete stochastic process learnable.
% Nevertheless, reinforcement learning suffers from extremely high time consumption and needs a large amount of data in the training phase, making it difficult to be used in the spanning tree process.
% First, motivated by the adaptive dropout network~\cite{ba2013adaptive}, we propose a simple strategy to make the spanning tree process learnable.
Firstly, we propose a simple strategy to make the spanning tree process learnable.
As shown in the right dashed rectangle of Fig.~\ref{fig:tf_v2}(b), we calculate the joint affinities by performing the element-wise production on pairwise similarities, which are generated from the input feature ${X}$ and the guided feature ${G}$, respectively.
And then, we use the joint affinities as the edge weights $\omega$ for the feature transformer.
This strategy creates a gradient tunnel between the guided feature and the output feature ${Y}$ to make the guided feature trainable utilizing the back-propagation algorithm~\cite{lecun2015deep}.

Moreover, a close random spanning tree algorithm is designed to replace the original minimum spanning tree in the training phase.
As illustrated in Alg.~\ref{alg:rst}, the proposed algorithm is a modification of the Boruvka's algorithm~\cite{eppstein1999spanning}, which replaces the minimum selection to the stochastic sampling (line 6 in Alg.~\ref{alg:rst}) according to the edge weights.
This algorithm has the ability to regularize the network and avoid falling into local optima, which is resulted from the greedy strategy of the minimum spanning tree.
For the reason that the input $\mathcal{G}$ is a 4-connected planar graph, the computational complexity of the proposed algorithm can be reduced to linear on the number of pixels by using the edge contraction operation~\cite{gross2005graph} (line 8 in Alg.~\ref{alg:rst}).
In the training phase, the weights of selected edges as well as the sampling distribution will be optimized.
Besides, in the inference phase, we still use the minimum spanning tree to keep the results deterministic.

\subsection{Network Architecture}

Based on the algorithms in Sec.~\ref{sec:ab} and Sec.~\ref{sec:lrst}, we propose a generic feature transform module, namely the Learnable Tree Filter V2 module (\textit{LTF-V2 module}).
As illustrated in Fig.~\ref{fig:tf_v2}(b), the LTF-V2 module is composed of a feature transformer and a learnable spanning tree process.
To highlight the effectiveness of the LTF-V2 module, we adopt a simple embedding operator to instantiate the function $f(\cdot)$ in the unary term, \textit{i.e.}, $f(x_i)=\mathrm{Sigmoid}(\pi {x_i}^\top)$, where $\pi \in\mathbb{R}^{1\times C}$.
In addition, following the design of the LTF-V1 module, we instantiate the pairwise function $S_{\mathcal{G}_T}(\cdot)$ of edge weights as the \textit{Euclidean distance} and adopt the \textit{grouping strategy} for the edge weights and the unary term. 
Specifically, the grouping strategy~\cite{xie2017aggregated} is adopted to split the input feature into specific groups along the channel dimension and aggregate them with different weights. The default group number is set to 16, and the detailed comparison is provided in the supplementary material.

Due to the linear computational complexity, the LTF-V2 module is highly efficient and can be easily embedded as a learnable layer into deep neural networks.
To this end, we propose two usages of the module for the encoder and the decoder, respectively.
The usage for the encoder is shown as the left part of Fig.~\ref{fig:tf_v2}(a), which embeds the LTF-V2 module between adjacent stages of the encoder. 
In this way, the resize operation is used on the low-level feature (\textit{e.g.}, the ``Stage n" in Fig.~\ref{fig:tf_v2}(a)) when the dimensions of the low-level feature and the high-level feature (\textit{e.g.}, the ``Stage n+1" in Fig.~\ref{fig:tf_v2}(a)) are inconsistent. 
Furthermore, the LTF-V2 module adopts the resized low-level feature as a guided feature ${G}$ to generate the spanning tree and transform the high-level feature.
Different from the encoder, the usage for the decoder is shown in the right diagram of Fig.~\ref{fig:tf_v2}(a). The LTF-V2 module is embedded in the decoder and relies on the corresponding low-level feature in the encoder to generate the spanning tree.
% Since the most of existing networks {\color{blue} xxx} have either encoder or decoder, these two configurations can be applied to a variety of applications.

In this paper, we conduct experiments on semantic segmentation and instance-aware tasks, \textit{i.e.}, object detection and instance segmentation.
For semantic segmentation, the ResNet~\cite{he2016deep} is adopted as the backbone with the naive decoder following the LTF-V1.
While for instance-aware tasks, we adopt the Mask R-CNN~\cite{he2017mask} with FPN as the decoder and ResNet~\cite{he2016deep}/ResNeXt~\cite{xie2017aggregated} as the encoder.

% \begin{figure}[t]
% \centering
% \includegraphics[width=8.5cm]{Fig/architecture.png}
% \caption{The diagram of the usages of proposed LTF-V2 modules. The left and the right correspond to the usages of the LTF-V2 module in the encoder and the decoder, respectively. The red lines are the resize operators for the low-level features (\textit{i.e.}, ``Stage n") to generate the guided feature ${G}$. The green circle denotes the LTF-V2 module, which generates the spanning tree by using the guided feature to transform the high-level feature (\textit{i.e.}, ``Stage n+1").}
% \label{fig:arch}
% \end{figure}

\begin{figure*}[t]
\centering
\includegraphics[width=0.8\linewidth]{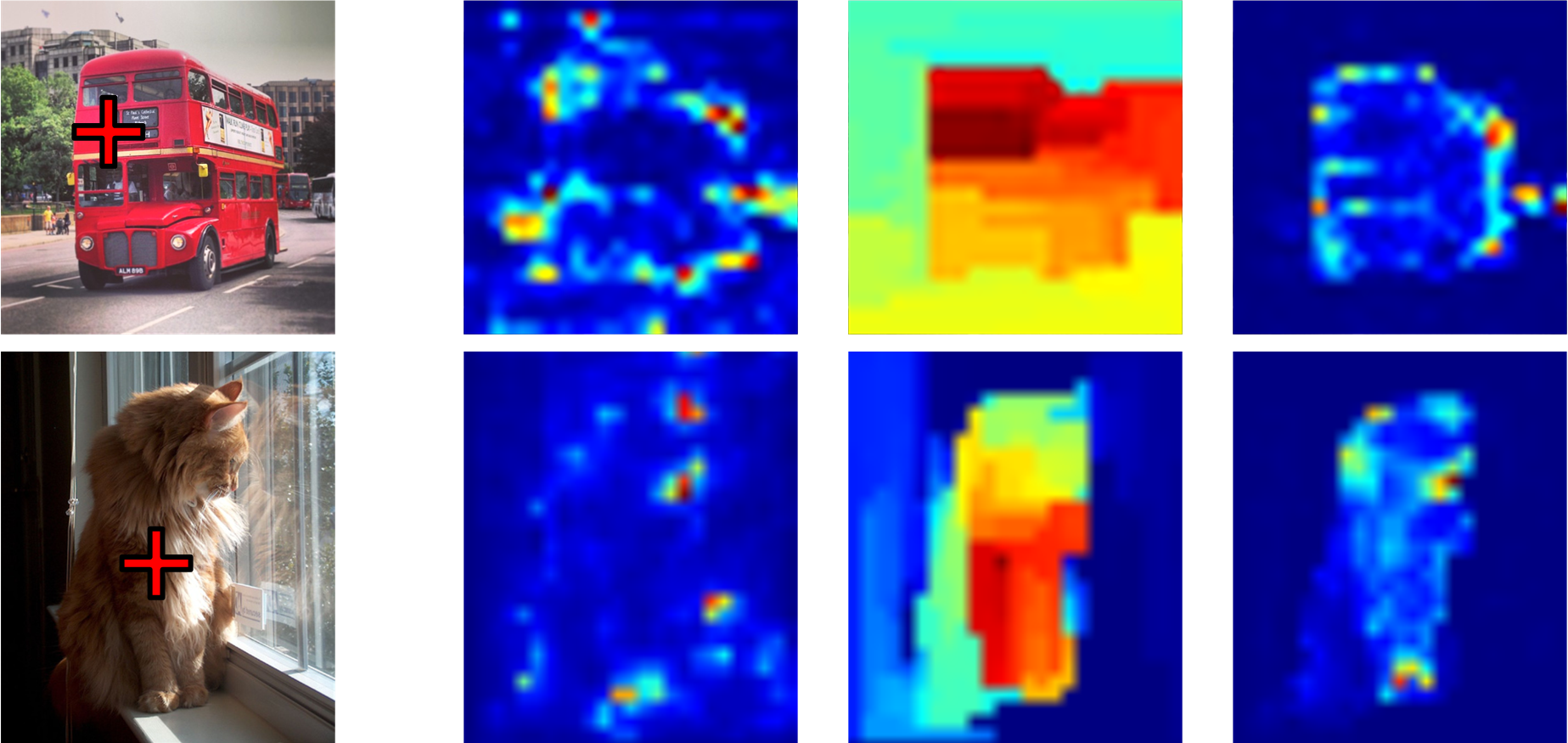}
% \subfigure[Affinity maps of unary term, pairwise term and LTF-V2 (left to right)]{
% \begin{minipage}[t]{\linewidth}
% \centering
% \includegraphics[width=0.8\linewidth]{Fig/coco_vis_mrf.png}
% % \caption{Affinity maps of different groups on LTF-V1 (top) and LTF-V2 (bottom).}
% \end{minipage}%
% }%
% \vspace{-5pt}
% \\
% \subfigure[Affinity maps of different groups on LTF-V1 (top) and LTF-V2 (bottom)]{
% \begin{minipage}[t]{\linewidth}
% \centering
% \includegraphics[width=0.8\linewidth]{Fig/coco_vis_comp.png}
% \caption{Affinity maps of different groups on LTF-V1 (top) and LTF-V2 (bottom).}
% \end{minipage}%
% }%
% \vspace{-10pt}
\caption{Visualization of the affinity maps in specific positions (marked by the \textit{red} cross in each image). The affinity maps (from left to right) are generated from the {\em unary term}, the {\em pairwise term} and the {\em LTF-V2}, respectively.
% The responses of different components and groups are illustrated in (a) and (b), respectively. 
% The figure (a) intuitively shows that the properties of unary and pairwise terms are complementary. The unary term can guide the transformer to focus on the \textit{coarsely} distant regions of interest and the pairwise term further \textit{refine} the regions to fit the detailed structures. With this property, as shown in (b), the feature diversity of the LTF-V2 module is obviously higher than that of the LTF-V1 module due to the relaxation of geometric constraints
% the LTF-V2 module can simultaneously focus on the distant regions and preserve the structures, which attributes to the properties of unary term and pairwise term, respectively.
}
\label{fig:coco_vis}
\end{figure*}

\section{Experiments}
To demonstrate the effectiveness of the proposed method, several experiments are conducted on \textit{COCO}~\cite{lin2014microsoft} for object detection/instance segmentation (instance-aware tasks) and \textit{Cityscapes}~\cite{cordts2016cityscapes} for semantic segmentation.
%on three fundamental vision tasks, \textit{i.e.}, object detection, instance segmentation and semantic segmentation.

% Below we present our experimental analysis on the
% MSCOCO dataset in Sec. 4.1, followed by our results on
% the ScanNet dataset in the Sec. 4.2. In both tasks, we first
% introduce implementation details of our method, and then
% report the ablative study and comparisons of quantitative
% results on each dataset.
% To evaluate the proposed method, we carry out experiments on three basic tasks, object detection/segmentation
% on COCO [21], image classification on ImageNet [5], and
% action recognition on Kinetics [17]. Experimental results
% demonstrate that the proposed GCNet generally outperforms both non-local networks (with lower FLOPs) and
% squeeze-excitation networks (with comparable FLOPs).

\subsection{Experiments in Instance-aware Tasks}

\subsubsection{Training Setting}
Following the configuration of Mask R-CNN in the Detectron2~\cite{wu2019detectron2} framework, we employ the FPN based decoder and a pair of 4-convolution heads for bounding box regression and mask prediction, respectively.
All the backbones are pre-trained on the ImageNet classification dataset~\cite{deng2009imagenet} unless otherwise specified.
In the training phase, input images are resized so that the shorter side is 800 pixels.
All the hyper-parameters are identical to the 1x schedule in the Detectron2 framework.
% Specifically, we fix parameters of the first two stages in the backbone and then jointly finetune the rest with detection and segmentation heads.
% All the experiments are trained on 8 GPUs with 2 images per GPU (effective mini-batch
% size of 16) for 90K iterations.
% The learning rate is initially set to 0.02 and then decreased by 10 at the 60K and 80K iterations. 
% All the models are optimized by using Synchronized SGD~\cite{krizhevsky2012imagenet} with a weight decay of 0.0001 and a momentum of 0.9.

% \begin{figure}[ht]
% \includegraphics[width=0.8\textwidth]{Fig/coco_vis.png}
% \centering
% \caption{Visualization of the affinity maps in specific positions (marked by the red cross in each image). The responses of different groups and components are illustrated in (a) and (b), respectively. The figure (a) intuitively shows that the diversity of the LTF-V2 module is significantly higher than that of the LTF-V1 module due to the relaxation of geometric constraints. The figure (b) reveals that the LTF-V2 module can simultaneously focus on the distant regions and preserve the structures, which attributes to the properties of unary term and pairwise term, respectively.}
% \label{fig:coco_vis}
% \end{figure}

\subsubsection{Ablation Study}
% To elaborate on the properties of the proposed methods, we conduct extensive ablation studies on COCO 2017 \textit{val} set.
% The experiments explore the equipped position and the stage, the group number of embedding space, and the spanning tree algorithm.
% The results on stronger backbones are provided to further demonstrate the effectiveness.
% To illustrate the mechanism, we give some visualizations of the tree filter modules in Fig.~\ref{fig:coco_vis}.

% \textbf{Instance-aware relations.}
% Due to the structure-preserving property, the proposed method is able to model instance-aware relations.
% As illustrated in Fig.~\ref{fig:coco_vis}~(b), given a reference position, the responses of the corresponding instance are partially activated.
% Even when a pair of instances overlap with each other, the proposed LTF-V2 module can still distinguish them.
% As shown in Tab.~\ref{tab:det_stage}, we infer that the instance-aware property is the reason why the LTF-V2 module is superior to the other related works.

\textbf{Components of the LTF-V2.}
Intuitively, as illustrated in Fig.~\ref{fig:coco_vis}, the properties of unary and pairwise terms are complementary. The unary term can guide the LTF-V2 to focus on the coarsely distant regions, and the pairwise term further refines the regions to fit the detailed structures. Furthermore, we give the qualitative comparisons on Tab.~\ref{tab:det_terms}, which demonstrates the balance of the two terms can work better than the single one. To reveal the impact of the data-dependent spanning tree, we adopt the uniformly random spanning tree~\cite{gross2005graph} for the pairwise term by default.

\begin{table*}[t!]
\centering
\caption{Comparisons among related works and the LTF-V2 module on different stages for COCO 2017 \textit{val} set. All experiments are conducted on the Mask-RCNN framework with 1x learning rate schedule. 
GCNet~\cite{cao2019gcnet} is used after each bottleneck, and other blocks are applied at the end of specific stages in the backbone.
The settings of all the models are configured as the original papers.}
% All the additional blocks are applied at the end of stage, except GCNet at end of each bottleneck, in the ResNet-50 backbone}
\resizebox{\linewidth}{24mm}{
\begin{tabular}{lccccccccc}
\toprule
Model & Stage & $\mathrm{AP}_{box}$ & $\mathrm{AP}_{box}^{50}$ & $\mathrm{AP}_{box}^{75}$ & $\mathrm{AP}_{seg}$ & $\mathrm{AP}_{seg}^{50}$ & $\mathrm{AP}_{seg}^{75}$ & \#FLOPs & \#Params \\
\midrule
ResNet-50 (1x)  & -         & 38.8 & 58.7 & 42.4 & 35.2 & 55.6 & 37.6 & 279.4B & 44.4M \\ \midrule
+Non-Local~\cite{wang2018non} & 4    & 39.5 & 59.6 & 42.7 & 35.6 & 56.7 & 37.6 & +10.67B & +2.09M \\
+CCNet~\cite{huang2019ccnet}  & 345  & 40.1 & 60.4 & 44.1 & 36.0 & 57.4 & 38.4 & +16.62B & +6.88M \\
+LatentGNN~\cite{zhang2019latentgnn}      & 345  & 40.6 & 61.3 & 44.5 & 36.6
& 58.1 & 39.2 & +3.59B & +1.07M \\
+GCNet~\cite{cao2019gcnet} & All  & 40.7 & 61.0 & 44.2 & 36.7
& 58.1 & 39.2 & +0.35B & +10.0M \\
\midrule
+LTF-V1 & 345  & 40.0 & 60.4 & 43.7 & 36.1 & 57.5 & 38.4 & +0.31B & +0.06M \\ 
\midrule
\multirow{4}*{+LTF-V2}     & 3    & 40.1 & 59.9 & 43.9 & 36.0 & 57.1 & 38.3 & +0.43B & +0.02M \\
                ~         & 4    & 40.6 & 61.0 & 44.4 & 36.6 & 58.2 & 39.0 & +0.26B & +0.04M \\
                ~          & 5    & 40.2 & 60.5 & 43.6 & 36.1 & 57.5 & 38.4 & +0.17B & +0.08M  \\
                ~          & 345  & \textbf{41.2} & \textbf{61.6} & \textbf{45.2} & \textbf{37.0} & \textbf{58.4} & \textbf{39.5} & +0.68B & +0.14M \\
\bottomrule
\end{tabular}}
\label{tab:det_stage}
\end{table*}
\begin{table}[t]
\begin{minipage}{0.45\linewidth}
\centering
\caption[t]{Comparison of different spanning tree algorithms for the LTF-V2 module on COCO 2017 \textit{val} set. \textbf{MST} and \textbf{LST} are the minimum spanning tree algorithm and the proposed learnable spanning tree algorithm, respectively. \textbf{Pretrain} indicates the pre-trained weights for the backbone.}
\vspace{6pt}
\resizebox{62mm}{11mm}{
\begin{tabular}{lcccc}
\toprule
Pretrain & MST & LST & $\mathrm{AP}_{box}$ & $\mathrm{AP}_{seg}$ \\
\midrule
\multirow{2}*{Scratch} & \cmark & \xmark & 29.4 & 26.7 \\
                    ~ & \xmark & \cmark & \textbf{30.3}& \textbf{27.6} \\
\midrule
\multirow{2}*{ImageNet~\cite{deng2009imagenet}} & \cmark & \xmark & 40.9 & 36.8\\
                    ~ & \xmark & \cmark & \textbf{41.2} & \textbf{37.0} \\
\bottomrule
\end{tabular}}
\label{tab:det_lst}
\end{minipage}
\quad
\begin{minipage}{0.5\linewidth}
\caption{Comparisons among different backbones for the Mask-RCNN framework on COCO 2017 \textit{val} set.}
\centering
\resizebox{63mm}{20mm}{
\begin{tabular}{lcc}
\toprule
Model & $\mathrm{AP}_{box}$ & $\mathrm{AP}_{seg}$ \\
\midrule
ResNet-101 (1x) & 40.7 & 36.6 \\ 
+LTF-V1          & 41.6 & 37.3 \\
+LTF-V2          & \textbf{42.5} & \textbf{38.0} \\
\midrule
ResNeXt-101 (1x) & 43.0 & 38.3 \\ 
+LTF-V1           & 43.8 & 39.0 \\
+LTF-V2           & \textbf{44.5} & \textbf{39.7} \\
\midrule
ResNeXt-101 + Cascade (1x) & 45.5 & 39.3 \\ 
+LTF-V1           & 46.2 & 40.0 \\
+LTF-V2           & \textbf{46.9} & \textbf{40.4} \\
\bottomrule
\end{tabular}}
\label{tab:det_backbones}
\end{minipage}
\end{table}

\textbf{Stages.}
We explore the effect when inserting the LTF-V2 module after the last layer of different stages. 
% As shown in Tab.~\ref{tab:det_stage}, the results demonstrate the effectiveness of our LTF-V2 module among stages3, stage4, and stage5.
% Specifically, when the module is inserted to the stage4, it achieves 1.8\% and 1.4\% absolute gains for $\rm{AP}^{box}$ and $\rm{AP}^{seg}$, respectively.
The performance in stage 4 is better than in stage 3 and stage 5, which reflects the advantage of using both semantic context and detailed structure.
% The difference among the three stages is also illustrated in the third row of Fig.~\ref{fig:coco_vis}, where the high-level stage contains more semantic cues and the low-level stage focuses more on the details.
Moreover, due to the efficiency of our proposed framework, we can further insert the LTF-V2 module into multiple stages to fuse multi-scale features.
In this way, as shown in Tab.~\ref{tab:det_stage}, our method achieves 41.2\% on $\rm{AP}^{box}$ and 37.0\% on $\rm{AP}^{seg}$, which has \textbf{2.4}\% and \textbf{1.8}\% absolute gains over the baseline on $\rm{AP}^{box}$ and $\rm{AP}^{seg}$, respectively.

\textbf{Spanning tree algorithm.}
To evaluate the proposed learnable spanning tree algorithm, we conduct ablation studies with the original minimum spanning tree algorithm.
% The grouping strategy is disabled to avoid its impact on the analysis of the results.
% The modules are embedded at the end of stage3, stage4, and stage5 in the encoder.
As shown in Tab.~\ref{tab:det_lst}, the results demonstrate the effectiveness of the proposed learnable spanning tree algorithm, which achieves consistent improvements on both localization and segmentation.
Especially {\em without} pre-training, the improvement is more prominent, which indicates that the proposed algorithm improves the robustness to alleviate the adverse impact of random initialization.

\textbf{Stronger backbones.}
To further validate the effectiveness, we evaluate the proposed LTF-V2 modules on stronger backbones.
As shown in Tab.~\ref{tab:det_backbones}, we adopt ResNet-101 or ResNeXt-101 as the backbone.
Following the same strategy, we insert the learnable tree filter modules at the end of stage3, stage4, and stage5, respectively.
The LTF-V2 module still achieves noticeable performance gains over stronger baselines.
Specifically, when using cascade strategy~\cite{cai2018cascade} and ResNeXt-101 as the backbone, we achieve 1.4\% and 1.1\% absolute gains over the baseline for $\rm{AP}^{box}$ and $\rm{AP}^{seg}$, respectively.

\subsection{Experiments on Semantic Segmentation}
\subsubsection{Training Setting}
For semantic segmentation, training images are randomly resized by 0.5 to 2.0$\times$ and cropped to 1024$\times$1024. 
The random flipping horizontally is applied for data augmentation.
Furthermore, we employ 8 GPUs for training, and the effective mini-batch size is 16.
Following conventional protocols~\cite{yu2018bisenet, yu2018learning, song2019learnable}, we set the initial learning rate to 0.01 and employ the ``poly" schedule (\textit{i.e.}, multiply $(1 - \frac{iter}{max\_iter})^{power}$ for each iteration) with $\mathrm{power}=0.9$.
All models are optimized by using synchronized SGD with the weight decay of 0.0001 and the momentum of 0.9.
% The experiments for semantic segmentation are conducted on the Cityscapes~\cite{cordts2016cityscapes} dataset, which is composed of 2975, 500, and 1525 images in \textit{train}, \textit{val}, and \textit{test} set, respectively.
For fair comparisons, we adopt the LTF-V2 modules in each stage of the decoder as the LTF-V1 does.

\subsubsection{Ablation Study}
As shown in Tab.~\ref{tab:seg_ablation}, we give quantitative comparisons between the LTF-V1 module and the LTF-V2 module.
The results show that both yield significant improvement over the baseline.
Specifically, without data augmentation, the LTF-V2 module achieves \textbf{4.5}\% and \textbf{1.5}\% absolute gains on mIoU over the baseline and the LTF-V1 module, respectively.
Moreover, we adopt multi-scale and flipping augmentations for testing.
The performance of the LTF-V2 module is further improved, which still attains \textbf{3.4}\% and \textbf{1.7}\% absolute gains on mIoU over the baseline and the LTF-V1 module, respectively.

\begin{table}[t!]
\centering
\caption{Comparisons among different components of the LTF-V2 module on COCO 2017 \textit{val} set. \textbf{Unary} and \textbf{Pairwise} are the components of the proposed MRF, where the pairwise term adopts the uniformly random spanning tree by default. \textbf{LST} is the proposed learnable spanning tree algorithm.}
\resizebox{\linewidth}{16mm}{
\begin{tabular}{lccccccccc}
\toprule
Model & Unary & Pairwise & LST & $\mathrm{AP}_{box}$ & $\mathrm{AP}_{box}^{50}$ & $\mathrm{AP}_{box}^{75}$ & $\mathrm{AP}_{seg}$ & $\mathrm{AP}_{seg}^{50}$ & $\mathrm{AP}_{seg}^{75}$  \\
\midrule
\multirow{6}*{ResNet-50 (1x)}  & \xmark & \xmark & \xmark & 38.8 & 58.7 & 42.4 & 35.2 & 55.6 & 37.6 \\
~ & \cmark & \xmark & \xmark & 39.7 & 59.7 & 43.5 & 35.8 & 56.8 & 38.1 \\
% ~ & \xmark & \cmark & \xmark & 39.9 & 60.6 & 43.6 & 35.9 & 57.1 & 38.4 \\
~ & \xmark & \cmark & \xmark & 39.6 & 59.5 & 43.3 & 35.7 & 56.6 & 37.9 \\
% ~ & \xmark & \cmark & \xmark & 39.3 & 59.1 & 42.8 & 35.4 & 56.2 & 37.7 \\
~ & \cmark & \cmark & \xmark & 40.7 & 61.2 & 44.8 & 36.7 & 58.2 & 39.3 \\
~ & \xmark & \cmark & \cmark & 40.2 & 60.7 & 43.8 & 36.2 & 57.4 & 38.8 \\
~ & \cmark & \cmark & \cmark & \textbf{41.2} & \textbf{61.6} & \textbf{45.2} & \textbf{37.0} & \textbf{58.4} & \textbf{39.5} \\
\bottomrule
\end{tabular}}
\label{tab:det_terms}
\end{table}
\begin{table}[ht]
\begin{minipage}{0.45\linewidth}
\centering
\caption[h]{The ablation studies conducted on Cityscapes \textit{val} set. \textbf{MS} and \textbf{Flip} denote adopting multi-scale and flipping augmentation for testing, respectively.}
\vspace{13pt}
\resizebox{64mm}{15mm}{
\begin{tabular}{lccc}
\toprule
Model & MS\&Flip & mIoU (\%) & mAcc (\%) \\
\midrule
ResNet-50   & \xmark & 72.9 & 95.5 \\ 
+LTF-V1      & \xmark & 75.9 & 95.8 \\
+LTF-V2      & \xmark & \textbf{77.4} & \textbf{96.0} \\
\midrule
ResNet-50   & \cmark & 75.5 & 95.9 \\ 
+LTF-V1      & \cmark & 77.2 & 96.0 \\
+LTF-V2      & \cmark & \textbf{78.9} & \textbf{96.1}\\
\bottomrule
\end{tabular}}
\label{tab:seg_ablation}
\end{minipage}
\quad
\begin{minipage}{0.5\linewidth}
\centering
\caption{Comparisons with state-of-the-art results
on Cityscapes \textit{test} set. Our model is trained with fine annotations only.}
\resizebox{65mm}{19mm}{
\begin{tabular}{lcc}
\toprule
Model & Backbone & mIoU (\%) \\ 
\midrule
% RefineNet~\cite{lin2017refinenet}      & ResNet-101   &  73.6 \\
% DSSPN~\cite{liang2018dynamic}          & ResNet-101   &  77.8 \\
PSPNet~\cite{zhao2017pyramid}          & ResNet-101   &  78.4 \\ 
% BiSeNet~\cite{yu2018bisenet}           & ResNet-101   &  78.9 \\
DFN~\cite{yu2018learning}              & ResNet-101   &  79.3 \\ 
% PSANet~\cite{zhao2018psanet}           & ResNet-101   &  80.1 \\
DenseASPP~\cite{yang2018denseaspp}     & DenseNet-161 &  80.6 \\ 
LTF-V1~\cite{song2019learnable}         & ResNet-101   &  80.8 \\ 
CCNet~\cite{huang2019ccnet}            & ResNet-101   &  81.4 \\ 
% BANet~\cite{ding2019boundary}          & ResNet-101   &  81.4 \\
DANet~\cite{fu2019dual}                & ResNet-101   &  81.5 \\
SPNet~\cite{hou2020strip}              & ResNet-101   &  82.0 \\
\midrule
Ours (LTF-V2)                           & ResNet-101   &  \textbf{82.1} \\
\bottomrule
\end{tabular}}
\label{tab:seg_sta}
\end{minipage}
\end{table}

\subsubsection{Comparison with State-of-the-arts}
To further improve the performance, we adopt a global average pooling operation and an additional ResBlocks~\cite{he2016deep} in the decoder, which follows the design in the LTF-V1 module.
Specifically, an extra global average pooling operator is inserted at the end of stage4 in the encoder.
In addition, multiple ResBlocks based on ``Conv3$\times$3" are added before each upsampling operator in the decoder.
Similar with the conventional protocols~\cite{yu2018bisenet, yu2018learning, song2019learnable}, we further finetune our model on both \textit{train} and \textit{val} sets.
As shown in Tab.~\ref{tab:seg_sta}, we achieve \textbf{82.1\%} mIoU result on Cityscapes \textit{test} set, which is superior to other state-of-the-art approaches with ResNet-101 backbone and only fine annotations.

\section{Conclusion}
% The previous tree filter module (LTF-V1 module) is able to preserve the detailed structure and achieves superior performance on semantic segmentation.
% However, due to the geometric constraint, the LTF-V1 is forced to focus on the features with close spatial distance.
In this paper, we first rethink the advantages and shortages of the LTF-V1 by reformulating it as a Markov Random Field. Then we present a learnable unary term to relax the geometric constraint and enable effectively long-range interactions.
Besides, we propose the learnable spanning tree algorithm to replace the non-differentiable one for an entirely learnable tree filter.
Extensive ablation studies are conducted to elaborate on the effectiveness and efficiency of the proposed method, which is demonstrated to bring significant improvements on both instance-aware tasks and semantic segmentation with negligible computational overhead. 
We hope that the perspective of Markov Random Field for context modeling can provide insights into future works, and beyond.

\section*{Broader Impact}
Context modeling is a powerful tool to improve the ability for feature representation, which has been widely applied in real-world scenarios, \textit{e.g.}, computer vision and natural language processing. 
The traditional tree filter~\cite{yang2015stereo} already has great impacts on many low-level computer vision tasks, owing to its structure-preserving property and high efficiency. This paper further releases its representation potential by relaxing the geometric constraint.
Specifically, our method provides a new perspective for context modeling by unifying the learnable tree filter with the Markov Random Field, which is further demonstrated to be effective in several vision tasks with negligible computational and parametric overheads.
These properties of our method have great potentials, which allow our method and principle to extend to other complex tasks with large-number nodes, {\em e.g.}, replacing the attention module of transformer for natural language processing and enhancing sequential representation for video analysis. 

\begin{ack}
This research was supported by National Key R\&D Program of China (No. 2017YFA0700800), National Natural Science Foundation of China (No. 61751401) and Beijing Academy of Artificial Intelligence (BAAI).
\end{ack}

{\small
\bibliographystyle{unsrt}
\bibliography{reference.bib}
}

\clearpage
\appendix

\section{Visualization}
In this section, as shown in Fig.~\ref{fig:res_coco} and Fig.~\ref{fig:res_city}, we present several visualization results of the baseline (w/o context block) and the proposed LTF-V2 module for instance-aware tasks and semantic segmentation, respectively.

\begin{figure}[H]
\centering
\includegraphics[width=0.99\textwidth]{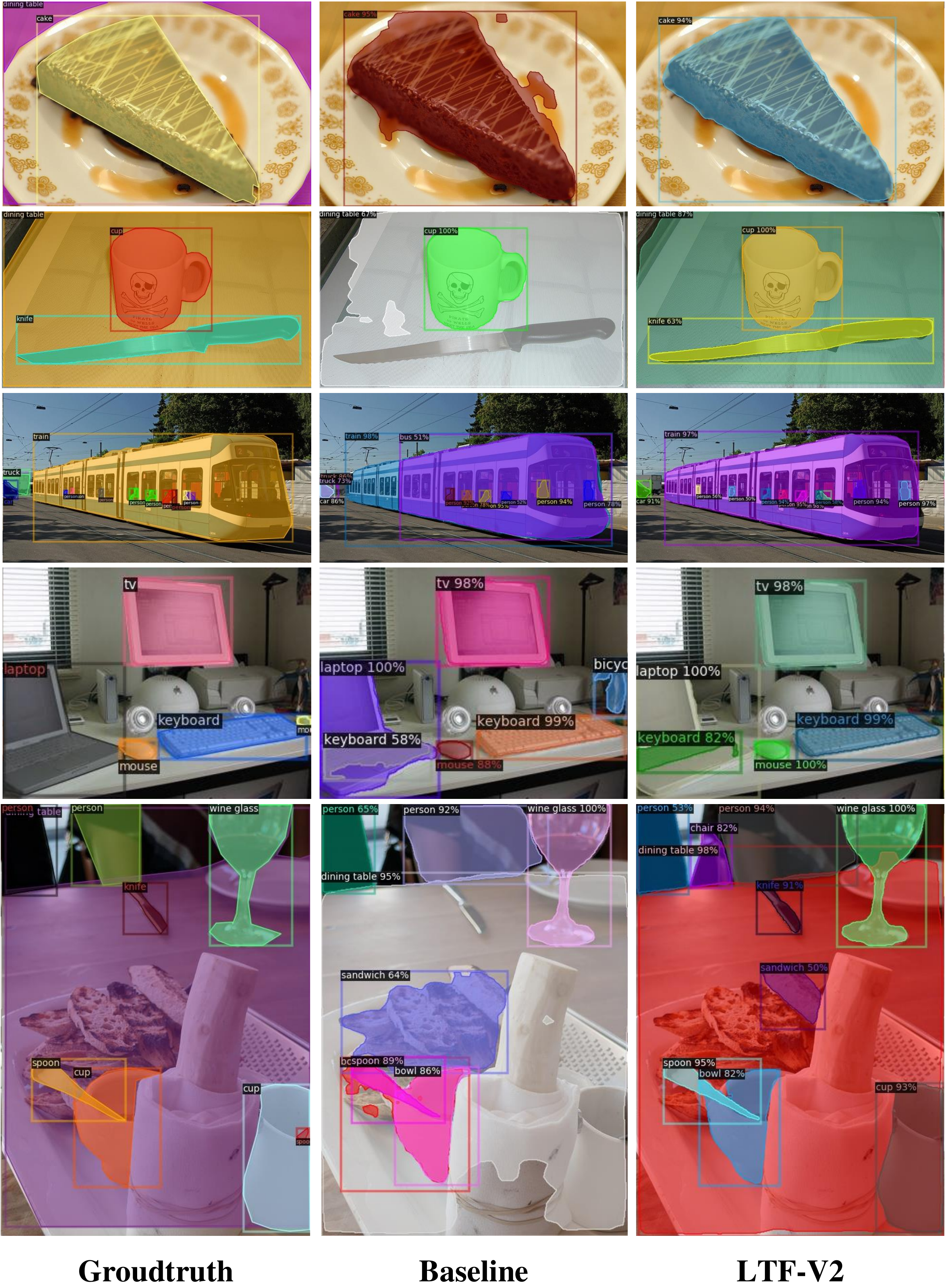}
\caption{Visualization results on COCO 2017 \textit{val} set. We use different colors to distinguish between instances that do \textit{not} represent categories. The results demonstrate the superiority of the LTF-V2 module in terms of both semantics and details}
\label{fig:res_coco}
\end{figure}

\clearpage

\begin{figure}[ht]
\centering
\includegraphics[width=\textwidth]{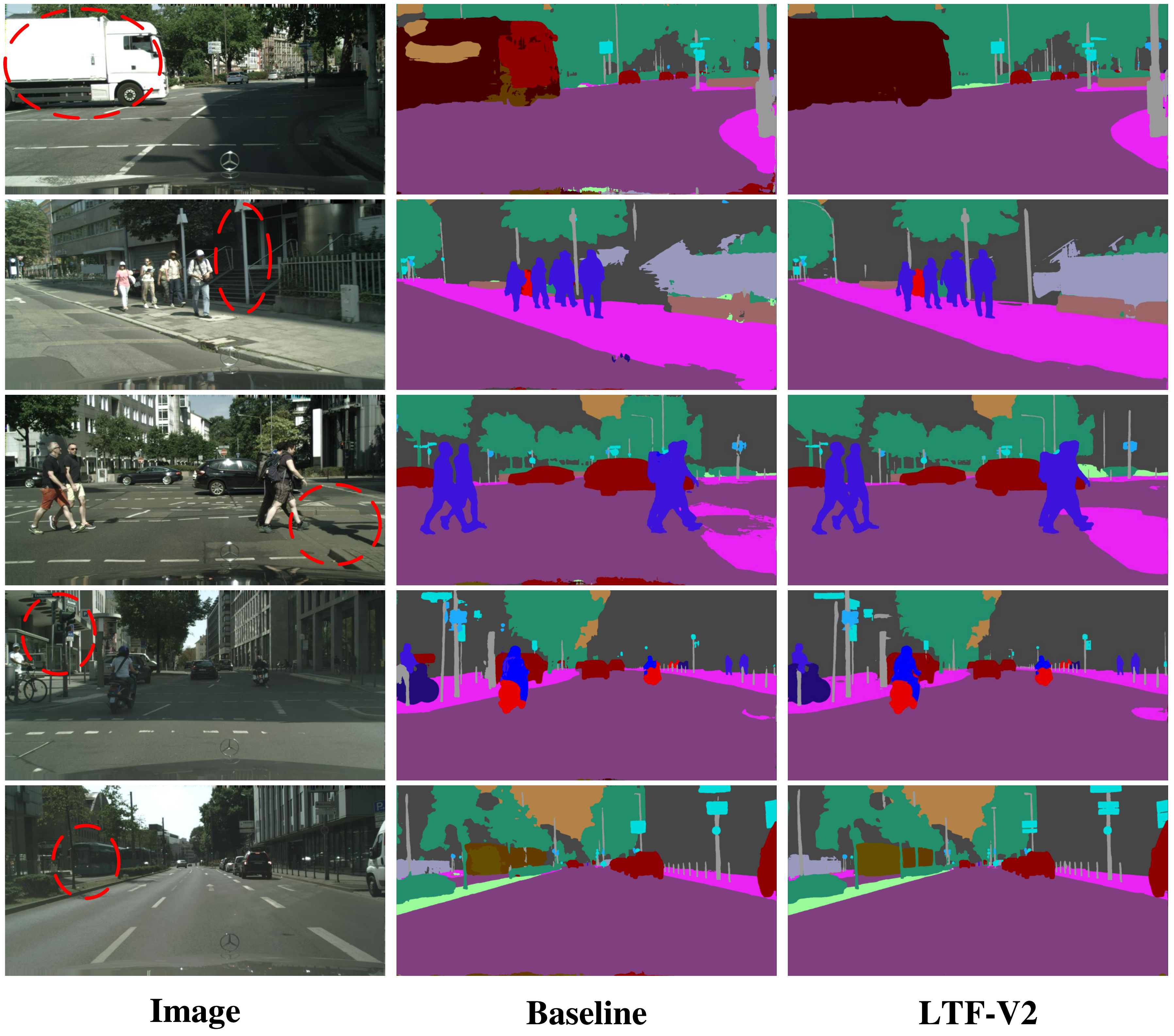}
\caption{Visualization results on Cityscapes \textit{val} set. The red dotted ellipses indicate the area of the major difference between the baseline and the LTF-V2 module}
\label{fig:res_city}
\end{figure}

\section{Algorithm Proof}
In this section, we present the detailed proofs for the claims about the learnable tree filter. Please note that the symbols follow the definitions in the main paper.

\begin{lemma}
Given a Markov Random Field on the tree $\mathcal{G}_T$, whose unary and pairwise terms are the Eq.~\ref{eq_a:phi} and the Eq.~\ref{eq_a:psi} respectively,  
% $\delta(\cdot)$ denotes the unit impulse function and ${\rm{Desc}}_{\mathcal{G}_T}(i, j)$ indicates the latent variables of descendants of node $j$ when node $i$ is the root. 
the marginal probability of latent variable satisfies $P_{\mathcal{G}_T}(h_i=j)=\frac{1}{z_i}S_{\mathcal{G}_T}(\mathbf{E}_{j,i})$.
\begin{equation}
\label{eq_a:phi}
    \phi_i(h_i, x_i)\equiv 1,
\end{equation}
\begin{equation}
\label{eq_a:psi}
    \psi_{i,j}(h_i, h_j):=\left\{\begin{matrix}
    \delta(h_i - h_j) & h_i\notin {\rm{Desc}}_{\mathcal{G}_T}(i, j) \\ 
    {\rm{exp}}(-\omega_{i,j}) \delta(h_i - h_j) & h_i\in {\rm{Desc}}_{\mathcal{G}_T}(i, j)
    \end{matrix}\right.
\end{equation}
\end{lemma}
\begin{proof}
To obtain the marginal probability of the Markov Random Field on an acyclic graph, we adopt the belief propagation algorithm as shown on Eq.~\ref{eq_a:b_i}, where $\mathcal{N}_i$ denotes the set of adjacent nodes of node $i$ in the tree. 
% Since the tree $\mathcal{G}_T$ is an acyclic graph, the belief propagation algorithm conducts exact inference.
\begin{equation}
\begin{aligned}
\label{eq_a:b_i}
    P_{\mathcal{G}_T}(h_i)&=\frac{1}{z_i}\phi_i(h_i, x_i)\prod_{\forall j\in \mathcal{N}_i}m_{j,i}(h_i),\\
    m_{j,i}(h_i)&=\sum_{\forall h_j\in \mathcal{V}}{\phi_j(h_j, x_j)\psi_{j, i}(h_j, h_i)\prod_{\forall k\in \mathcal{N}_j\setminus i}m_{k, j}(h_j).}
\end{aligned}
\end{equation}
For $h_i=i$, the marginal probability is
\begin{equation}
\begin{aligned}
    &P_{\mathcal{G}_T}(h_i=i)=\frac{1}{z_i}\prod_{\forall j\in \mathcal{N}_i}m_{j,i}(i)\\
    % &=\frac{1}{z_i}\prod_{\forall j\in \mathcal{N}_i}({\sum_{\forall h_j\in \mathbf{X}}{\psi_{j, i}(h_j, x_i)\prod_{\forall k\in \mathcal{N}_j\setminus i}m_{k, j}(h_j)}})\\
    &=\frac{1}{z_i}\prod_{\forall j\in \mathcal{N}_i}({\sum_{\forall h_j\in \mathcal{V}}{\delta(h_j - i)\prod_{\forall k\in \mathcal{N}_j\setminus i}m_{k, j}(h_j)}})\\
    &=\frac{1}{z_i}\prod_{\forall j\in \mathcal{N}_i}{\prod_{\forall k\in \mathcal{N}_j\setminus i}m_{k, j}(i)}\\
    &=\frac{1}{z_i}\prod_{\forall j\in \mathcal{N}_i}{\prod_{\forall k\in \mathcal{N}_j\setminus i}\prod_{\forall u\in \mathcal{N}_k\setminus k}\cdots 1}\\
    &=\frac{1}{z_i}.
\end{aligned}
\end{equation}
For $h_i=u$ and $u\neq i$, the marginal probability is
\begin{equation}
\begin{aligned}
    &P_{\mathcal{G}_T}(h_i=u)=\frac{1}{z_i}\prod_{\forall j\in \mathcal{N}_i \cap \{p\mid x_u\in {\rm Desc}(i,p)\}}m_{j,i}(u)\\
    &=\frac{1}{z_i}{\sum_{\forall h_j\in \mathcal{V}}{\psi_{j, i}(h_j, u)\prod_{\forall k\in \mathcal{N}_j\setminus i}m_{k, j}(h_j)}}\\
    &=\frac{1}{z_i}{\sum_{\forall h_j\in \mathcal{V}}{\rm exp}(-\omega_{j, i}){\delta(h_j - u)\prod_{\forall k\in \mathcal{N}_j\setminus i}m_{k, j}(h_j)}}\\
    &=\frac{1}{z_i}{\rm exp}(-\omega_{j, i})\prod_{\forall k\in \mathcal{N}_j\setminus i}m_{k, j}(u)\\
    &=\frac{1}{z_i}{\rm exp}(-\omega_{j, i}){\rm exp}(-\omega_{k, j})\cdots \prod_{\forall v\in \mathcal{N}_u\setminus par(u)}m_{v, u}(u)\\
    &=\frac{1}{z_i}{\rm exp}(-\omega_{j, i}){\rm exp}(-\omega_{k, j})\cdots {\rm exp}(-\omega_{u, par(u)})\\
    &=\frac{1}{z_i}\prod_{\forall (k,m)\in \mathbf{E}_{j,i}}{{\rm exp}(-\omega_{k, m})}.
\end{aligned}
\end{equation}
Therefore, $P_{\mathcal{G}_T}(h_i=j)=\frac{1}{z_i}S_{\mathcal{G}_T}(\mathbf{E}_{j,i})$.
% \begin{equation}
% \begin{aligned}
%     &P_{\mathcal{G}_T}(h_i=j)=\frac{1}{z_i}S_{\mathcal{G}_T}(\mathbf{E}_{j,i})
% \end{aligned}
% \end{equation}

\end{proof}

\begin{lemma}
Given a Markov Random Field on the tree $\mathcal{G}_T$, whose unary and pairwise terms are the Eq.~\ref{eq_a:phi} and the Eq.~\ref{eq_a:psi}, respectively. When denoting node $i$ as the root of the tree $\mathcal{G}_T$ and node $v$ (distant) as one of the descendants node of node $u$ (nearby), the marginal probability of latent variable satisfies $P_{\mathcal{G}_T}(h_i=v)\leq  P_{\mathcal{G}_T}(h_i=u)$.
\end{lemma}

\begin{proof}
Since the edge distance satisfies $\omega_{k, m}\geq 0$, the marginal probability of latent variable satisfies
\begin{equation}
\begin{aligned}
    &P_{\mathcal{G}_T}(h_i=v)=\frac{1}{z_i}S_{\mathcal{G}_T}(\mathbf{E}_{v,i})\\
    &=\frac{1}{z_i}S_{\mathcal{G}_T}(\mathbf{E}_{u,i})S_{\mathcal{G}_T}(\mathbf{E}_{v,u})\\
    &=P_{\mathcal{G}_T}(h_i=u)\prod_{\forall (k,m)\in \mathbf{E}_{v,u}}{{\rm exp}(-\omega_{k, m})}\\
    &\leq P_{\mathcal{G}_T}(h_i=u).
\end{aligned}
\end{equation}
% \begin{equation}
% \begin{aligned}
%     &P_{\mathcal{G}_T}(h_i=x_v)=\frac{1}{z_i}\prod_{\forall (k,m)\in \mathbf{E}_{v,i}}{{\rm exp}(-\omega_{k, m})}\\
%     &=\frac{1}{z_i}\prod_{\forall (k,m)\in \mathbf{E}_{u,i}}{{\rm exp}(-\omega_{k, m})}\prod_{\forall (k,m)\in \mathbf{E}_{v,u}}{{\rm exp}(-\omega_{k, m})}\\
%     &=P_{\mathcal{G}_T}(h_i=x_u)\prod_{\forall (k,m)\in \mathbf{E}_{v,u}}{{\rm exp}(-\omega_{k, m})}\\
%     &\leq P_{\mathcal{G}_T}(h_i=x_u).
% \end{aligned}
% \end{equation}
\end{proof}

\begin{lemma}
Given a Markov Random Field on the tree $\mathcal{G}_T$, whose unary and pairwise terms are the Eq.~\ref{eq_a:f_x} and the Eq.~\ref{eq_a:psi}, respectively. When denoting node $i$ as the root of the tree $\mathcal{G}_T$ and node $v$ (distant) as one of the descendants of node $u$ (nearby), there exists a specific pair of  $f(\cdot)$ and $\beta$, such that the marginal probability of latent variable satisfies $P_{\mathcal{G}_T}(h_i=v)>  P_{\mathcal{G}_T}(h_i=u)$.
\begin{equation}
\label{eq_a:f_x}
\phi_i(h_i, x_i) := \left\{\begin{matrix}
f(x_i) & h_i=i\\ 
{\rm{exp}}(-\beta) & h_i\neq i
\end{matrix}\right.
\end{equation}
\end{lemma}

\begin{proof}
For $h_i=v$, the marginal probability is
\begin{equation}
\begin{aligned}
    &P_{\mathcal{G}_T}(h_i=v)=\frac{1}{z_i}{{\rm{exp}}(-\beta)}^{|\mathbf{E}_{v,i}|}f(x_v)S_{\mathcal{G}_T}(\mathbf{E}_{v,i})\\
    &=\frac{1}{z_i}{{\rm{exp}}(-\beta)}^{|\mathbf{E}_{u,i}|}S_{\mathcal{G}_T}(\mathbf{E}_{u,i})f(x_v){{\rm{exp}}(-\beta)}^{|\mathbf{E}_{v, u}|}S_{\mathcal{G}_T}(\mathbf{E}_{v,u})\\
    &=P_{\mathcal{G}_T}(h_i=u)\frac{f(x_v)}{f(x_u)}\prod_{\forall (k,m)\in \mathbf{E}_{v,u}}{{\rm exp}(-\omega_{k, m}-\beta)}.
\end{aligned}
\end{equation}
When $\frac{f(x_v)}{f(x_u)}\prod_{\forall (k,m)\in \mathbf{E}_{v,u}}{{\rm exp}(-\omega_{k, m}-\beta)} > 1$, the marginal probability of latent variable satisfies $P_{\mathcal{G}_T}(h_i=v) > P_{\mathcal{G}_T}(h_i=u)$.
\end{proof}

\section{Runtime}
Due to trees are acyclic graphs, we can adopt a well-designed dynamic programming algorithm to reduce the computational complexity to linear w.r.t vertex number.
Besides, in order to improve the efficiency on GPU devices, we parallelize the algorithm along with batches, channels, and nodes of the same depth. As shown in Fig.~\ref{fig:benchmark}, the empirical runtime of our method is far less than that of the Non-Local network~[12], which uses an extremely sophisticated tensor library.
We believe that the efficiency of LTF-V2 can be further improved through device-oriented code optimization.

\begin{figure}[ht]
\centering
\includegraphics[width=0.9\textwidth]{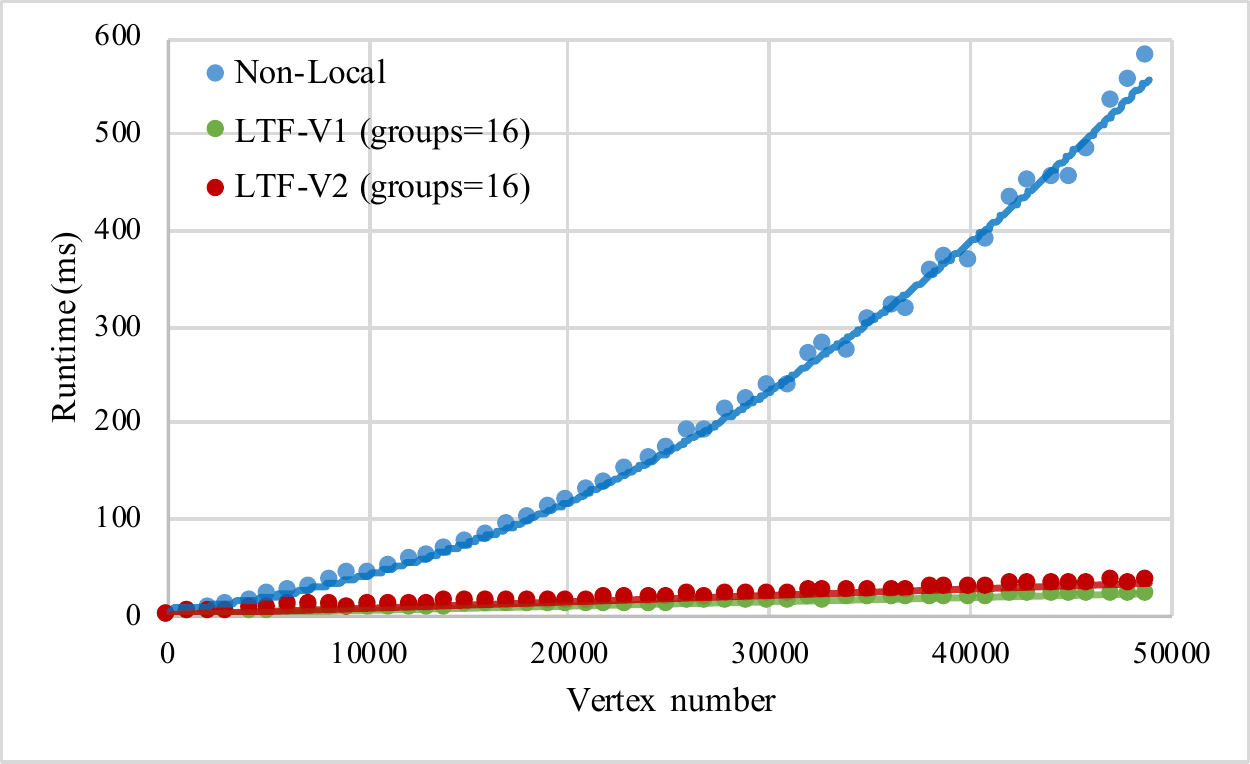}
\caption{Comparisons of the runtime on a Tesla V100 GPU among LTF-V1, LTF-V2 and Non-Local network. The number of feature channels is set to 512}
\label{fig:benchmark}
\end{figure}

\clearpage
\section{Additional Ablation Study}

\begin{table*}[h]
\centering
\caption{Comparisons among different positions and different groups on COCO 2017 \textit{val} set when using ResNet-50 as our backbone. \textbf{Enc} and \textbf{Dec} denote embedding the blocks in the backbone and the FPN, respectively}
\resizebox{\linewidth}{20mm}{
\begin{tabular}{l|c|c|ccc|ccc|cc}
\toprule
Model & Position & Groups & $\mathrm{AP}_{box}$ & $\mathrm{AP}_{box}^{50}$ & $\mathrm{AP}_{box}^{75}$ & $\mathrm{AP}_{seg}$ & $\mathrm{AP}_{seg}^{50}$ & $\mathrm{AP}_{seg}^{75}$ & \#FLOPs & \#Params \\
\midrule
\multirow{3}*{+LTF-V1} & Dec & 16 & 39.6 & 60.0 & 43.4 & 35.8 & 56.9 & 38.0 & +0.45B & +0.02M \\
                    ~ & Enc & 1 & 39.6 & 59.6 & 43.1 & 35.6 & 56.5 & 37.9 & +0.09B & +0.01M\\
                    ~ & Enc & 16 & 40.0 & 60.4 & 43.7 & 36.1 & 57.5 & 38.4 & +0.31B & +0.06M \\
\midrule
\multirow{6}*{+LTF-V2} & Dec & 16 & 39.9 & 60.1 & 43.6 & 36.0 & 57.1 & 38.4 & +1.23B & +0.06M \\
                    ~ & Enc & 1 & 40.2 & 60.5 & 44.0 & 36.3 & 57.5 & 38.8 & +0.13B &  +0.01M\\
                    ~ & Enc & 4 & 40.9 & 61.3 & 44.5 & 36.8 & 58.4 & 39.4 & +0.24B &  +0.04M\\ 
                    ~ & Enc & 8 & 40.9 & 61.4 & 44.3 & 36.8 & 58.4 & 39.2 & +0.39B &  +0.07M\\ 
                    ~ & Enc & 16 & \textbf{41.2} & \textbf{61.6} & \textbf{45.2} & \textbf{37.0} & 58.4 & \textbf{39.5} & +0.68B & +0.14M \\
                    ~ & Enc & 32 & 41.1 & 61.6 & 44.8 & 37.0 & \textbf{58.7} & 39.4 &  +1.28B & +0.29M \\
\bottomrule
\end{tabular}}
\label{tab:det_pos_group}
\end{table*}
\begin{figure}[H]
\subfigure[Affinity maps of LTF-V1 (top) and LTF-V2 (bottom) with different groups]{
    \includegraphics[width=0.9\linewidth]{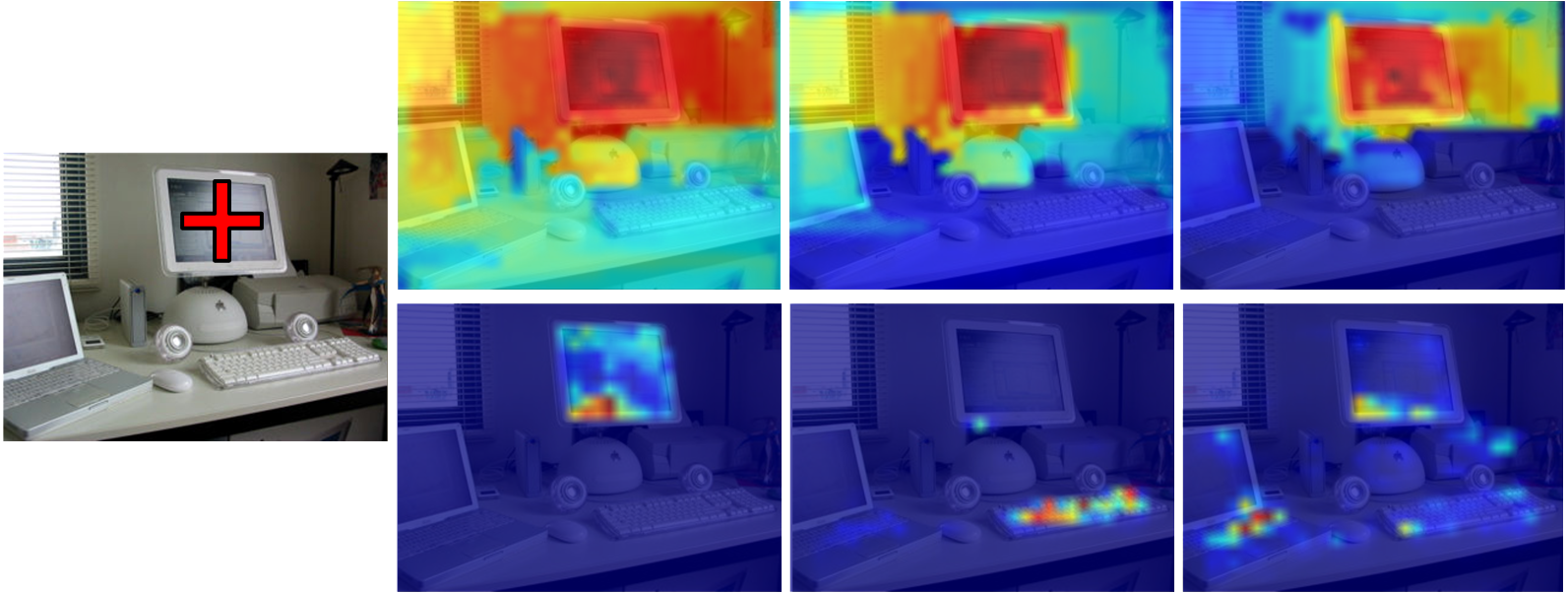}
}\\
\subfigure[Affinity maps of LTF-V2 for different instances]{
    \includegraphics[width=0.9\linewidth]{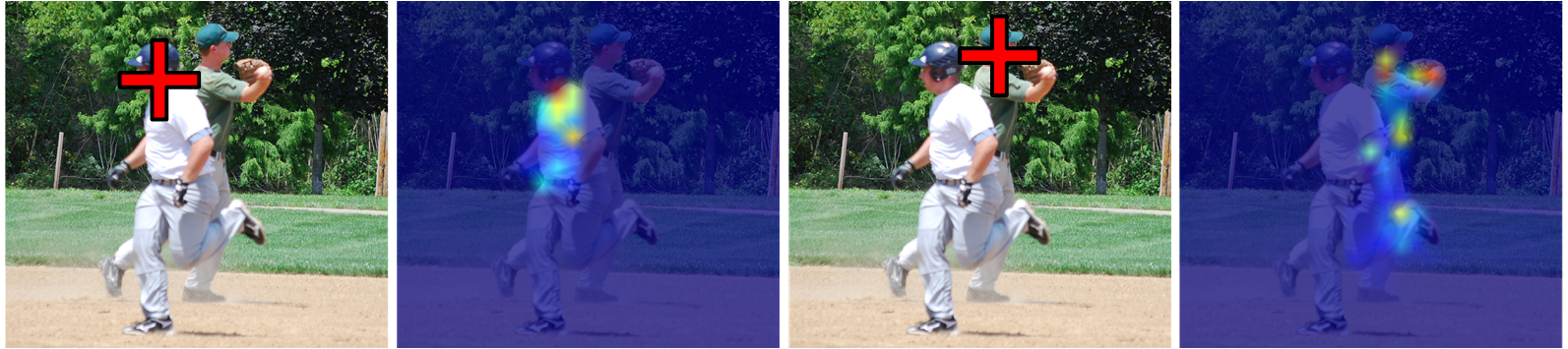}
}\\
\subfigure[Affinity maps of LTF-V2 in different stages]{
    \includegraphics[width=0.9\linewidth]{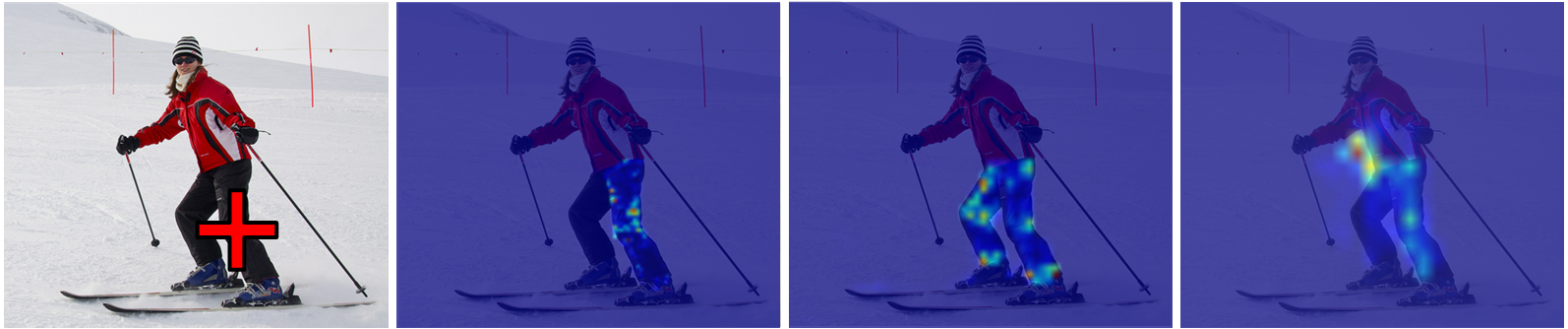}
}\\
\centering
\caption{Visualization of the affinity maps in specific positions (marked by the red cross in each image). The responses of different groups, instances, and stages are illustrated in (a), (b), and (c), respectively. The figure (a) intuitively shows that the diversity of the LTF-V2 module is significantly higher than that of the LTF-V1 module due to the relaxation of geometric constraints. The figure (b) reveals that the LTF-V2 module is aware of the structure of instance, even when the instances of the same category are overlapping. Besides, the heat maps (from left to right) in (c) correspond to stage3, stage4, and stage5 in the encoder, respectively. }
\label{fig_a:coco_vis}
\end{figure}

\end{document}